\documentclass[journal,draftclsnofoot,onecolumn,10pt,romanappendices]{IEEEtran}
\usepackage{cite}
\usepackage{amsmath,epsfig,amssymb,algorithm,algpseudocode,amsthm,cite,url,amsfonts,xcolor}
\usepackage{pgfplots}
\usepackage{pgfplotstable}
\pgfplotsset{compat=1.14}
\usepackage{tikz}

\def\BibTeX{{\rm B\kern-.05em{\sc i\kern-.025em b}\kern-.08em
    T\kern-.1667em\lower.7ex\hbox{E}\kern-.125emX}}
\usepackage{hyperref}
\usepackage{enumerate}
    \DeclareMathOperator{\tr}{tr}
       \DeclareMathOperator{\diag}{diag}
        \DeclareMathOperator{\var}{var}   
\theoremstyle{plain}
\newtheorem{theorem}{Theorem}
\newtheorem{assumption}{Assumption}

\newtheorem{proposition}[theorem]{Proposition}

\theoremstyle{definition}

\theoremstyle{remark}
\newtheorem*{remark}{Remark}

\newtheoremstyle{specialcasestyle}{1mm}{1mm}{\upshape}{}{\bfseries\upshape}{.}{0mm}{}
\theoremstyle{specialcasestyle}

\newcommand{\figref}[1]{Fig.~\protect\ref{#1}}

\newcommand{\bmu}{\boldsymbol{\mu}}
\newcommand{\bSig}{\boldsymbol{\Sigma}}

\newcommand{\bu}{{\bf u}}
\newcommand{\bg}{{\bf g}}

\newcommand{\bV}{{\bf V}}

      \newcommand{\bH}{{\bf H}}
            
         \newcommand{\bD}{{\bf D}}
         
                  \newcommand{\bN}{{\bf N}}

      \newcommand{\bx}{{\bf x}}
            \newcommand{\bv}{{\bf v}}

                    \newcommand{\be}{{\bf e}}
           \newcommand{\by}{{\bf y}}
          \newcommand{\bw}{{\bf w}}
        
      \newcommand{\bz}{{\bf z}}
        
               \newcommand{\bX}{{\bf X}}
      
      \newcommand{\bB}{{\bf B}}
           \newcommand{\bI}{{\bf I}}
    
       \newcommand{\bC}{{\bf C}}

                    \newcommand{\bM}{{\bf M}}

                \newcommand{\bZ}{{\bf Z}}
                  \newcommand{\bE}{{\bf E}}
                     \newcommand{\bOmega}{{\boldsymbol\Omega}}
                   \newcommand{\asto}{\overset{a. s.}\longrightarrow }
                            \newcommand{\dto}{\overset{d}\longrightarrow }

\usepackage{placeins}
\usepackage{float}
\usepackage{tabularx}
\usepackage{tkz-euclide,subfigure}
\usepackage{lmodern}

\begin{document}

\title{
\huge{High-Dimensional Quadratic Discriminant Analysis\\ under Spiked Covariance Model}}

\author{Houssem Sifaou,~\IEEEmembership{Student Member,~IEEE}, Abla Kammoun,~\IEEEmembership{Member,~IEEE}~ and~Mohamed-Slim~Alouini,~\IEEEmembership{Fellow,~IEEE}
\thanks{H. Sifaou, A. Kammoun and M.S. Alouini are with the Computer, Electrical, and Mathematical Sciences and Engineering (CEMSE) Division, KAUST, Thuwal, Makkah Province, Saudi Arabia (e-mail: houssem.sifaou@kaust.edu.sa, abla.kammoun@kaust.edu.sa, slim.alouini@kaust.edu.sa).}
}
\maketitle

\begin{abstract}
Quadratic discriminant analysis (QDA) is a widely used classification technique that generalizes the linear discriminant analysis (LDA) classifier to the case of distinct covariance matrices among classes. For the QDA classifier to yield high classification performance, an accurate estimation of the covariance matrices is required. Such a task becomes all the more challenging in high dimensional settings, wherein the number of observations is comparable with the feature dimension. A popular way to enhance the performance of QDA classifier under these circumstances is to regularize the covariance matrix,  giving the name regularized QDA (R-QDA) to the corresponding classifier. In this work, we consider the case in which the population covariance matrix has a spiked covariance structure, a model that is often assumed in several applications. Building on the classical QDA, we propose a novel quadratic classification technique, the parameters of which are chosen such that the fisher-discriminant ratio is maximized. Numerical simulations show that the proposed classifier not only outperforms the classical R-QDA for both synthetic and real data but also requires lower computational complexity, making it suitable to high dimensional settings.
\end{abstract}


\maketitle

\section{Introduction}

Classification is among the most typical examples of supervised learning techniques. When the data is normally distributed with common covariance matrices across classes, linear discriminant analysis (LDA) is known to be the optimal classifier in terms of misclassification rate minimization. In the case of different covariances across classes, it has recently been shown that the use of LDA does not enable to leverage the information on the differences between covariance matrices \cite{Elkhalil2017b}. Under such circumstances, it can be more advisable to employ the quadratic discriminant analysis (QDA), which turns out to be the optimal classifier under Gaussian data and known statistics. In practical scenarios, the covariance matrices and the means associated with each class are not perfectly known. They are often estimated based on the available training data for which the class label associated with each observation is provided. If the number of training samples $n$ and their dimensions $p$ are commensurable, a situation widely met in numerous applications such as medical imaging \cite{Friston}, functional data analysis \cite{Ramsay}, meteorology and oceanography \cite{Preisendorfer}, many estimators of the covariance matrices such as the sample covariance matrix become highly inaccurate. A typical extreme scenario corresponds to the case   $n<p$, in which the sample covariance matrix becomes singular, and as such, cannot be used as a plug-in estimator of the covariance matrix since the QDA classifier involves the computation of the inverse covariance matrix. To get around this issue, it was proposed to use instead, a regularized covariance matrix estimator that linearly shrinks through the use of a scalar regularization parameter the sample covariance matrix towards identity \cite{Friedman1989}. The corresponding classifier is referred to as regularized QDA (R-QDA). This regularization appoach has been used succefully in several applications \cite{ye2006regularized,xiong2018mathcal,bian2017early}. However, QDA and R-QDA remain widely unused in high-dimensional settings, being very sensitive to the estimation quality of the covariance matrix \cite{Zollanvari2015}.  

In this work, we consider a high-dimensional setting in which the number of observations is assumed to scale with their dimensions. We further assume that the population covariance matrix associated with each class is a low-rank perturbation of a scaled identity; that is, it is isotropic except for a finite number of symmetry-breaking directions.
Such a model is used in many real applications such as detection\cite{ZHAO19861}, electroencephalogram (EEG) signals\cite{Davidson2009,FAZLI20112100}, and financial econometrics\cite{Passemier2017,KRITCHMAN2008}, and is known in the random matrix theory terminology as the spiked covariance model.  
Based on this model, we propose to employ for each class a  parametrized covariance matrix estimator following the same model as the population covariance matrix. The parameters correspond to the largest eigenvalues, which are optimized to maximize the classifier performance. More specifically, by leveraging tools from random matrix theory, we compute the asymptotic Fisher ratio in the regime $n$ and $p$ growing to infinity at the same pace. Closed-form expressions of the optimal parameters that maximize the asymptotic Fisher ratio are then provided. The approach consisting of exploiting the spiked structure of the covariance matrix has mainly been considered in signal processing applications \cite{yang} and \cite{donoho}. It has only recently been used for the classification problem in our work in   \cite{Sifaou2018a,JMLR:v21:19-428}, wherein a similar approach is applied to find an improved LDA classifier under the spiked covariance model assumption. Considering a QDA based classifier is needed when the covariance matrices between classes are different. It is also more challenging since it involves an involved quadratic statistic, the statistical properties of which are much harder to characterize.   

The proposed classifier is compared with the regularized QDA (R-QDA) classifier \cite{Friedman1989} using both real and synthetic data. The proposed classifier outperforms the classical R-QDA classifier while requiring less computational complexity. As shown next in the paper,  the proposed classifier involves a statistic that avoids computing the inverse of the covariance matrix. Moreover, since the parameters are obtained in closed-form, it avoids the grid search or the cross-validation approach needed to determine the optimal regularization parameter of the R-QDA classifier\cite{Friedman1989}.

The remainder of this paper is organized as follows. In the next section, a brief overview of QDA and R-QDA classifiers is provided. Section \ref{Improved_QDA} details the steps of the design of our proposed classifier. The performance of the proposed classifier is studied in section \ref{Numerical_Simulations}, and some concluding remarks are drawn in section \ref{conclusion}.

\subsection{Notations}
Throughout this work, boldface lower case is used for denoting column vectors, $\bx$, and upper case for matrices, $\bX$. $\bX^T$ denotes the transpose. Moreover, $\bI_p$, $\boldsymbol{0}_p$ and $\boldsymbol{1}_p$ denote the identity matrix, the all-zero vector and all-one vector of size $p$ respectively. $\left|\bX\right|$ and $\tr\left(\bX\right)$ denote the determinant and the trace of $\bX$ respectively. $\left\{x_j\right\}_{j=1}^r$ is used to denote the row vector with entries $x_j$ whereas $\|.\|$ is used to denote the $\ell_2$-norm. The almost sure convergence and the convergence in distribution of random variables will be denoted as $\asto$ and $\overset{d}\longrightarrow$ receptively.

\section{ Quadratic Discriminant Analysis}
Consider $\bx_1,\cdots,\bx_n$ observations of size $p$ belonging to two different classes $\mathcal{C}_0$ and $\mathcal{C}_1$ with $n_i$ observations belonging to class $\mathcal{C}_i$. For notational convenience, we denote by $\mathcal{T}_i$ the set of indexes of the observations belonging to class $\mathcal{C}_i$.
We assume that $\bx_\ell \in \mathcal{C}_i, i\in\{0,1\}$, is drawn from a Gaussian distribution with mean $\bmu_i$ and covariance $\bSig_i$. In this work, a 'spiked model' is assumed for the covariance matrices. Under this assumption, for $i\in\left\{0,1\right\}$,  $\bSig_i$ is written as: 
\begin{equation}
\bSig_i=\sigma_i^2\bI_p+ \sigma_i^2 \sum_{j=1}^{r_i}\lambda_{j,i}\bv_{j,i}\bv_{j,i}^T,   
\label{eq:sig}
\end{equation}
where $\sigma_i^2>0$, $\lambda_{1,i} \geq \cdots,\geq \lambda_{r_i,i}>0$ and $\bv_{1,i},\cdots,\bv_{r_i,i}$ are orthonormal.
\begin{remark}
The assumed model of the covariance matrices is encountered in many real applications such as detection\cite{ZHAO19861}, EEG signals\cite{Davidson2009,FAZLI20112100}, and financial econometrics\cite{Passemier2017,KRITCHMAN2008}.\end{remark}
The starting point of our work is the classical QDA classifier whose discriminant function is given by:
\begin{equation}
\begin{aligned}
W^{\rm QDA}(\bx)&=\eta^{\rm QDA}-\frac{1}{2} (\bx-\bmu_0)^T \bSig_0^{-1} (\bx-\bmu_0) +\frac{1}{2} (\bx-\bmu_1)^T \bSig_1^{-1} (\bx-\bmu_1),
\end{aligned}
\label{discriminant_function_QDA}
\end{equation}
where $\eta^{\rm QDA}=-\frac{1}{2}\log\frac{|\bSig_0|}{|\bSig_1|}-\log\frac{\pi_1}{\pi_0}$ and $\pi_i$ is the prior probability for class $\mathcal{C}_i$. An observation $\bx$ is classified to $\mathcal{C}_0$ if the discriminant function $W^{\rm QDA}(\bx)$ is positive and to class $\mathcal{C}_1$ otherwise.
In practice, the mean vectors and covariance matrices are unknown and are usually replaced by their empirical estimates. For notational convenience, we define the sample mean and the sample covariance matrix of class $i \in \left\{0,1\right\}$, respectively as:
\begin{align*}
&\overline\bx_i=\frac{1}{n_i} \sum_{\ell\in\mathcal{T}_i}  \bx_\ell,\\
&\hat\bSig_i=\frac{1}{n_i-1} \sum_{\ell\in\mathcal{T}_i}  (\bx_\ell -\overline\bx_i)(\bx_\ell -\overline\bx_i)^T.
\end{align*}
It is the case in many real data sets that the dimension of the observations is of the same order of magnitude if not higher than their numbers, which makes the sample covariance matrix $\hat\bSig_i$ ill-conditioned. 
To overcome this issue, ridge estimators of the inverse of the covariance
matrix are used \cite{hastie01statisticallearning,Zollanvari2015}:
\begin{equation}
\bH_i=\left( \bI_p+ \gamma \hat \bSig_i \right)^{-1}, \ \  \gamma>0.
\end{equation}
Replacing $\bSig_i$ by $\bH_i$ into \eqref{discriminant_function_QDA} yields the R-QDA classifier, the statistic of which is given by:
\begin{equation}
\begin{aligned}
\hat W^{\rm R-QDA}(\bx)&=\eta^{\rm R-QDA}-\frac{1}{2} (\bx-\overline\bx_0)^T \bH_0^{-1} (\bx-\overline\bx_0) +\frac{1}{2} (\bx-\overline\bx_1)^T \bH_1^{-1} (\bx-\overline\bx_1),
\end{aligned}
\label{discriminant_function_RQDA}
\end{equation}
where $\eta^{\rm R-QDA}=-\frac{1}{2}\log\frac{|\bH_1|}{|\bH_0|} -\log\frac{\pi_1}{\pi_0}$.
The classification error of R-QDA corresponding to class $i$ can be written as,
\begin{align*}
\epsilon^{\rm R-QDA}_i&=\mathbb{P}\left[(-1)^i \hat W^{\rm R-QDA}(\bx) <0 | \bx\in \mathcal{C}_i\right],
\end{align*}
The global classification error is given by,
\begin{equation}
\epsilon^{\rm R-QDA}=\pi_0\epsilon^{\rm R-QDA}_0+\pi_1\epsilon^{\rm R-QDA}_1.
\label{classificationError}
\end{equation}
 The optimal parameter of R-QDA classifier $\gamma^*$, that minimizes the global classification error, is generally computed by comparing the performance of a few candidate  values using a cross-validation method \cite{Friedman1989}.

\section{ Improved QDA}
\label{Improved_QDA}
\subsection{Proposed classification rule}

In this section, we propose an improved QDA classifier that leverages the structure of the covariance matrix model in \eqref{eq:sig}. For simplicity, we assume that $\sigma_i^2$ and $r_i$ are perfectly known. In practice, there exist several efficient algorithms in the literature for the estimation of these parameters. For more details, we refer the reader to the following works \cite{KRITCHMAN2008,Johnstone2009,Ulfarsson2008,Passemier2017}. 

Let $ \hat \bSig_i =\sum_{j=1}^{p} s_{j,i} \bu_{j,i}\bu_{j,i}^T,$ be  the eigenvalue decomposition of the sample covariance matrix corresponding to class $i$,
with $s_{j,i}$ is the $j$-th largest eigenvalue of $ \hat \bSig_i$ and $\bu_{j,i}$ its corresponding eigenvector. We look for an inverse covariance matrix estimator that possesses the same eigenvector basis. It can be thus written as:
$$
\hat \bC_i^{-1}=\sum_{j=1}^{p} t_{j,i}\bu_{j,i}\bu_{j,i}^T,
$$
where $t_{j,i}$ are some parameters to be designed. In accordance with the covariance matrix model in \eqref{eq:sig}, it is natural to set $t_{p-r,i}=\cdots=t_{p,i}=1/\sigma_i^2$. Such operation allows to shrink the covariance matrix estimator towards the structure described by \eqref{eq:sig}, giving it the name of  a shrinkage estimator \cite{Daniels2001}.
Thus, the inverse of the covariance matrix can be estimated as,
\begin{equation}
\hat \bC_i^{-1}=\frac{1}{\sigma_i^2}\left(\bI_p+\sum_{j=1}^{r_i} w_{j,i}  \bu_{j,i}\bu_{j,i}^T\right),
\label{cov_estimate}
\end{equation}
where $w_{j,i}=\sigma_i^2t_{j,i}-1$. In the sequel, we work with $w_{j,i}$ as the considered optimization variables. For notational convenience, we define $\bw=[w_{1,1},\cdots,w_{r_1,1},w_{1,0},\cdots,w_{r_0,0}]^T$. 
Our analysis relies on an asymptotic analysis of the behavior of the proposed QDA classifier. The asymptotic regime that is considered in our work is described in the following assumption: 
\begin{assumption} Throughout this work, we assume that, for $i\in\left\{0,1\right\}$, 
\begin{itemize}
\item[] (i) $n_i,p \asto \infty$, with fixed ratio $c_i=p/n_i$.
\item[] (ii) $r_i$ is fixed and $\lambda_{1,i}>\cdots>\lambda_{r_i,i}>\sqrt{c_i}$, independently of $p$ and $n_i$.
\item[] (iii) The spectral norm of $\bSig_i$, $\|\bSig_i\|$ are bounded, that is $\|\bSig_i\|=O(1)$.
\item[] (iv) The mean difference vector $\boldsymbol{\mu}\triangleq \boldsymbol{\mu}_1-\boldsymbol{\mu}_0$ has a bounded Euclidean norm, that is $\|\boldsymbol{\mu}\| =O(1)$.
\item[] (v) $\sigma_i^2=O(1)$ and $\sigma_0^2-\sigma_1^2=O(1/p)$.
\end{itemize}
\label{asymptotic_regime_assump}
\end{assumption}
\begin{remark}
\begin{itemize}
\item Assumption {\it (i)} is a key assumption that is generally in the framework of the theory of large random matrices. 
\item Assumption {\it (ii)} is fundamental in our analysis since it guarantees, as per standard results from random matrix theory, the one-to-one mapping between the sample eigenvalues $s_{j,i}$ and the unknown $\lambda_{j,i}$. In fact, when $\lambda_{j,i}>\sqrt{c_i}$, $\lambda_{j,i}$ can be consistently estimated using $s_{j,i}$ as we will see later. In the case where $\lambda_{j,i}\leq \sqrt{c_i}$, the relation between $s_{j,i}$ and $\lambda_{j,i}$ no longer holds and $\lambda_{j,i}$ cannot be estimated \cite{baik2005,Couillet2011}. 
\item Assumption {\it (v)} is a technical assumption, under which $\tr\left(\bSig_1-\bSig_0\right)=O(1)$. Moreover, from \eqref{eq:sig}, it ensures that the low-rank perturbation in $\bSig_i$ has a non-negligible contribution in $\tr\left(\bSig_1-\bSig_0\right)$. This is a key assumption that is needed for the parameter vector ${\bf w}$ to be asymptotically relevant for the classification.  
\end{itemize}
\end{remark}
Using the proposed covariance estimator, the discriminant function associated with the proposed classifier is given as:
\begin{equation}
\begin{aligned}
\hat W^{\rm Imp-QDA}(\bx)=&\eta-\frac{1}{2} (\bx-\bmu_0)^T \hat\bC_0^{-1} (\bx-\bmu_0) +\frac{1}{2} (\bx-\bmu_1)^T \hat\bC_1^{-1} (\bx-\bmu_1),
\end{aligned}
\label{discriminant_function_impQDA}
\end{equation}
where $\eta$ accounts for an additional bias; the way it is selected will be shown later. %
Let ${\bf x}$ be a testing observation belonging to class $\mathcal{C}_i$. Then, ${\bf x}=\bmu_i+\bSig_i^{\frac{1}{2}}{\bf z}$ with ${\bf z}\sim\mathcal{N}({\bf 0},{\bf I}_p)$. The classification error corresponding to class $\mathcal{C}_i$ can be written as,
\begin{align}
\epsilon^{\rm Imp-QDA}_i&=\mathbb{P}\left[(-1)^i \hat W^{\rm Imp-QDA}(\bx) <0 | \bx\in \mathcal{C}_i\right],\\
&=\mathbb{P}\left[(-1)^iY_i(\hat\bC_0,\hat\bC_1) < 0 | \bz \sim \mathcal{N}(\boldsymbol{0},\bI_p)\right],
\end{align}
where
\begin{equation}
Y_i(\hat\bC_0,\hat\bC_1)=\bz^T\bB_i\bz+2\by_i^T\bz -\xi_i,
\label{Y_i}
\end{equation}   
with
{\small
\begin{align*}
 \bB_i=&\bSig_i^{\frac{1}{2}} \left(\hat\bC_1^{-1}-\hat\bC_0^{-1}\right) \bSig_i^{\frac{1}{2}},\\
\by_i=& \bSig_i^{\frac{1}{2}} \left[\hat\bC_1^{-1} (\bmu_i-\overline\bx_1)- \hat\bC_0^{-1} (\bmu_i-\overline\bx_0)\right],\\
\xi_i=& -2\eta+ (\bmu_i-\overline\bx_0)^T\hat\bC_0^{-1}(\bmu_i-\overline\bx_0)-(\bmu_i-\overline\bx_1)^T\hat\bC_1^{-1}(\bmu_i-\overline\bx_1),
\end{align*}
\begin{proposition} Under the conditions ${\it (i)}$, ${\it (iii)}$ and ${\it (v)}$ of Assumption 1, we have
\begin{align*}
 Y_i(\hat\bC_0,\hat\bC_1) - \tilde Y_i\dto 0
\end{align*}
where
\begin{align*}
\tilde Y_i=p\sigma_i^2\left(\frac{1}{\sigma_1^2}-\frac{1}{\sigma_0^2}\right)+\nu_i+2\by_i^T\bz-\xi_i,
\end{align*}
with $\nu_i= \frac{1}{\sigma_1^2}\sum_{j=1}^{r_1} w_{j,1}( \tilde\bz_i^T\bu_{j,1})^2 -\frac{1}{\sigma_0^2} \sum_{j=1}^{r_0} w_{j,0} ( \tilde\bz_i^T  \bu_{j,0})^2$ and $\tilde\bz_i=\bSig_i^{\frac{1}{2}}\bz$.
\label{dist_conv}
\end{proposition}
It entails from Proposition \ref{dist_conv} that the asymptotic behavior of $Y_i(\hat\bC_0,\hat\bC_1)$ corresponds to that of a linear combination of a chi-squared and normal distributions. 
 For illustration, we plot in \figref{dist_plot} the empirical distributions of $Y_0(\hat\bC_0,\hat\bC_1)$ and $Y_1(\hat\bC_0,\hat\bC_1)$ built based on several testing vectors drawn from $\mathcal{C}_0$ and $\mathcal{C}_1$. Unfortunately. the distribution of ${Y}_i$ does not have closed form expressions, which makes the analysis of the misclassification rate cumbersome. 
\begin{figure}[]
\centering
\includegraphics[scale=0.65]{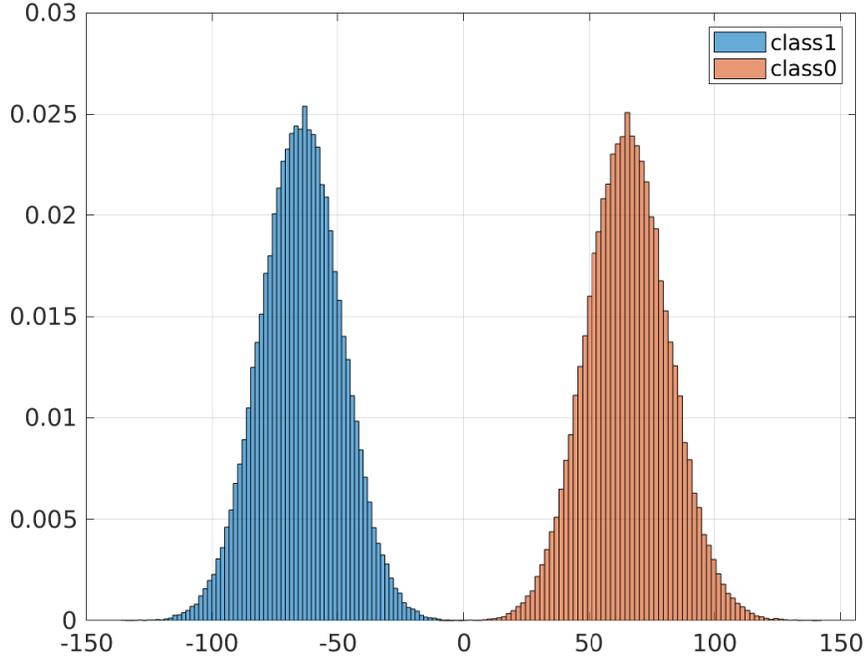}
\caption{Histograms of $ Y_0(\hat\bC_0,\hat\bC_1)$ and $ Y_1(\hat\bC_0,\hat\bC_1)$ with synthetic data with $\sigma_0^2=\sigma_1^2=1$, $r_0=r_1=3$, $ \lambda_{1,0}=\lambda_{1,1}=4$,  $ \lambda_{2,0}=\lambda_{2,1}=3$, $ \lambda_{3,0}=\lambda_{3,1}=2$ and $\bmu_0=-\bmu_1=\frac{4}{\sqrt{p}}\boldsymbol{1}_p$.}
\label{dist_plot}
\end{figure}
\subsection{Parameter optimization}

In this section, we present a possible setting of the parameter vector ${\bf w}$. Since the misclassification rate cannot be characterized in closed-form, we propose instead for tractability to maximize the Fisher ratio metric. Such a metric quantifies the separability between the two classes, by measuring the ratio of the separation between the means to the variance within classes and has been fundamental in the design of the Fisher discriminant analysis (FDA) based classifier.  
Under our setting, the square root of the Fisher-Ratio \cite{FISHER1936} associated with the classifier in \eqref{discriminant_function_QDA} is given by:
$$
\rho(\bw)=\frac{\left| m_0(\bw)-m_1(\bw) \right|}{\sqrt{v_0(\bw)+v_1(\bw)}},
$$
where for $i\in\left\{0,1\right\}$, $m_i(\bw)$ and $v_i(\bw)$ are respectively the mean and the variance of $\tilde Y_i$ with respect to $\bz$, given by:
\begin{align*}
m_i(\bw)&=p\sigma_i^2\left(\frac{1}{\sigma_1^2}-\frac{1}{\sigma_0^2}\right)+\mathbb E \nu_i-\xi_i,\\
v_i(\bw)&= \var(\nu_i)+4\by_i^T\by_i.
\end{align*}
where we have used the fact that $\nu_i$ and $\by_i^T\bz$ are uncorrelated. Following the same methodology in the design of FDA, we propose to select ${\bf w}$ that solves the following optimization problem:
$$
\bw^\star=\underset{\bw} {\rm arg max } \ \ \ \rho(\bw).
$$
 The optimization cannot be performed at this stage since $m_i({\bw})$ and $v_i(\bw)$ involve unknown quantities such as $\lambda_{j,i}$ and $\bv_{j,i}$ that appear in $\bSig_i$. To overcome this issue, we resort to techniques from random matrix theory which allows us to compute deterministic equivalents of $m_i(\bw)$ and $v_i(\bw)$. Using these deterministic equivalents, the unknown quantities $\lambda_{j,i}$ and $\bv_{j,i}$ can be consistently estimated by some observable quantities under the asymptotic regime defined in assumption \ref{asymptotic_regime_assump}. Before presenting the deterministic equivalents of $m_i(\bw)$ and $v_i(\bw)$, we shall first define the following quantities
 \begin{equation}
  \begin{aligned}
 & \alpha_i=\frac{\|\boldsymbol{\mu}\|^2}{\sigma_i^2}, \ i=0,1\\
& a_{j,i}=\frac{1-c_i/\lambda_{j,i}^2}{1+c_i/\lambda_{j,i}}, \ j=1,\cdots,r_i, \ i=0,1\\
&b_{j,i}=\frac{\bmu^T \bv_{j,i} \bv_{j,i}^T\bmu}{\|\bmu\|^2 }, \ j=1,\cdots,r_i, \ i=0,1\\
&\psi_{\ell,j,1,0}=\psi_{j,\ell,0,1}=\bv_{\ell,1}^T \bv_{j, 0}, \ \ell=1,\cdots,r_1, j=1,\cdots,r_0\\
& \phi_{j,0}=1+a_{j,0}\sum_{\ell=1}^{r_1}\lambda_{\ell,1}\psi_{\ell,j,1,0}^2,\  j=1,\cdots,r_0 \\
&\phi_{j,1}=1+a_{j,1}\sum_{\ell=1}^{r_0}\lambda_{\ell,0} \psi_{j,\ell,1,0}^2, \  j=1,\cdots,r_1
\end{aligned}
 \label{phi_ij}
 \end{equation}
Moreover, we shall assume that $\bmu^T\bu_{j,i}>0$ and $\bmu^T\bv_{j,i}>0$ for $i =0,1, \ j=1,\cdots,r_i $. This assumption, which is  needed to simplify the presentation of the results, can be made without loss of generality since eigenvectors are defined up to a sign.
\begin{theorem} Under the asymptotic regime defined in Assumption \ref{asymptotic_regime_assump}, we have
 \begin{equation}
m_i(\bw) - \overline m_i(\bw) \asto 0,
\label{mu_conv}
 \end{equation}
 \begin{equation}
v_i(\bw) - \overline v_i(\bw) \asto 0,
\label{v_conv}
 \end{equation}
with
\begin{align}
 \overline m_i(\bw) &=2\eta+c_1-c_0 +p\left(\frac{\sigma_i^2}{\sigma_1^2}-\frac{\sigma_i^2}{\sigma_0^2}\right) \label{mi_bar} \\ &+(-1)^i\alpha_{\tilde i}+\bg_i^T\bw, \nonumber  \\
\overline v_i(\bw) &= 4 \left(\bw^T\bE_i\bw+2\be_i^T\bw+b_i \right),\label{vi_bar}
\end{align}
where $\tilde i=1-i$ and 
{
\begin{align*}
\bg_0\!&=\!\left[\!\left\{\alpha_1a_{j,1}b_{j,1}+\frac{\sigma_0^2}{\sigma_1^2} \phi_{j,1}\right\}_{j=1}^{r_1},\left\{  -1-\lambda_{j,0}a_{j,0} \right\}_{j=1}^{r_0} \right]^T\!, \\
\bg_1\!&=\!\left[\!\left\{1+\lambda_{j,1}a_{j,1} \right\}_{j=1}^{r_1},-\left\{\alpha_0a_{j,0}b_{j,0}+\frac{\sigma_1^2}{\sigma_0^2} \phi_{j,0}\right\}_{j=1}^{r_0}\right]^T\!, 
\end{align*}
\begin{align*}
b_0&=\alpha_1\frac{\sigma_0^2}{\sigma_1^2}\left[1+\sum_{j=1}^{r_0}\lambda_{j,0}b_{j,0}\right]+c_1\frac{\sigma_0^4}{\sigma_1^4}+c_0,\\
b_1&=\alpha_0\frac{\sigma_1^2}{\sigma_0^2}\left[1+\sum_{j=1}^{r_1}\lambda_{j,1}b_{j,1}\right]+c_0\frac{\sigma_1^4}{\sigma_0^4}+c_1,
\end{align*}
\begin{align*}
\be_0\!&=\!\frac{\alpha_1\sigma_0^2}{\sigma_1^2}\!\left[\!\left\{\!a_{j,1}b_{j,1}\!+\!\sum_{\ell=1}^{r_0}\lambda_{\ell,0}a_{j,1}\sqrt{b_{j,1}b_{\ell,0}} \psi_{j,\ell,1,0}\!\right\}_{j\!=\!1}^{r_1}\!,\!\boldsymbol{0}_{r_0}\!\right]^T\\
\be_1\!&=\!\frac{\alpha_0\sigma_1^2}{\sigma_0^2}\!\left[\!\boldsymbol{0}_{r_1}, \!\left\{\!a_{j,0}b_{j,0}\!+\!\sum_{\ell=1}^{r_1}\lambda_{\ell,1}a_{j,0}\sqrt{b_{j,0}b_{\ell,1}} \psi_{j,\ell,0,1}\!\right\}_{\!j=\!1}^{r_0}\!\right]^T
\end{align*}
\begin{align*}
\bE_0&=\begin{bmatrix}
\tilde \bD_0+\bM_0&   \bN_0 \\
 \bN_0^T&  \bD_0
\end{bmatrix}, \   \    \
\bE_1=\begin{bmatrix}
 \bD_1 & \bN_1\\
 \bN_1^T&\tilde \bD_1+\bM_1
\end{bmatrix}, 
\end{align*}
with $\bD_i \in \mathbb{R}^{r_i\times r_i}$, $\tilde \bD_i,\bM_i\in \mathbb{R}^{r_{\tilde i}\times r_{\tilde i}}$ and $\bN_i\in  \mathbb{R}^{r_1\times r_0}$ defined as,
\begin{align*}
\bD_i&=\frac{1}{2}\diag\left\{  (1+\lambda_{j,i}a_{j,i})^2 \right\}_{j=1}^{r_i},\\
\tilde\bD_i&=\diag\left\{\ \frac{\sigma_i^4}{\sigma_{\tilde i}^4} \frac{\phi_{j,\tilde i}^2}{2}+\frac{\sigma_i^2}{\sigma_{\tilde i}^2}\alpha_{\tilde i}a_{j,\tilde i}b_{j,\tilde i} \right\}_{j=1}^{r_i},\\
 [N_i]_{\ell,j}&= -\frac{1}{2}\frac{\sigma_i^2}{\sigma_{\tilde i}^2}{(1+\lambda_{\ell,i})^2}a_{\ell,i}a_{j,\tilde i}\psi_{\ell,j,i,\tilde i}^2,\\
  [M_i]_{j,k}&=\alpha_{\tilde i}\frac{\sigma_i^2}{\sigma_{\tilde i}^2}a_{j,\tilde i}a_{k,\tilde i}\sqrt{b_{j,\tilde i}b_{k,\tilde i}}\sum_{\ell=1}^{r_i}\lambda_{\ell,i}\psi_{\ell,j,i,\tilde i}\psi_{\ell,k,i,\tilde i}.
\end{align*}
}
\label{m_v_conv}
\end{theorem}
\begin{remark}
Using item (v) of Assumption 1, the expressions in Theorem \ref{m_v_conv} can be further simplified by leveraging the fact that $\frac{\sigma_1^2}{\sigma_0^2}\to 1$. However, when handling real data sets, we observed that working with the non-simplified expressions may lead to better performances, due to a possible inaccuracy of item (v) in Assumption 1. This is the reason why in our simulations we worked with the expressions of Theorem \ref{m_v_conv}, which can be further simplified by substituting  $\frac{\sigma_1}{\sigma_0}$ and $\frac{\sigma_0}{\sigma_1}$ by 1. In doing so, we obtain the following simplified expressions which we provide below for the sake of completeness:

{
\begin{align*}
\bg_0\!&=\!\left[\!\left\{\alpha_1a_{j,1}b_{j,1}+\phi_{j,1}\right\}_{j=1}^{r_1},\left\{  -1-\lambda_{j,0}a_{j,0} \right\}_{j=1}^{r_0} \right]^T\!, \\
\bg_1\!&=\!\left[\!\left\{1+\lambda_{j,1}a_{j,1} \right\}_{j=1}^{r_1},-\left\{\alpha_0a_{j,0}b_{j,0}+\phi_{j,0}\right\}_{j=1}^{r_0}\right]^T\!, 
\end{align*}
\begin{align*}
b_0&=\alpha_1\left[1+\sum_{j=1}^{r_0}\lambda_{j,0}b_{j,0}\right]+c_1+c_0,\\
b_1&=\alpha_0\left[1+\sum_{j=1}^{r_1}\lambda_{j,1}b_{j,1}\right]+c_0+c_1,
\end{align*}
\begin{align*}
\be_0\!&=\!{\alpha_1}\!\left[\!\left\{\!a_{j,1}b_{j,1}\!+\!\sum_{\ell=1}^{r_0}\lambda_{\ell,0}a_{j,1}\sqrt{b_{j,1}b_{\ell,0}} \psi_{j,\ell,1,0}\!\right\}_{j\!=\!1}^{r_1}\!,\!\boldsymbol{0}_{r_0}\!\right]^T\\
\be_1\!&=\!{\alpha_0}\!\left[\!\boldsymbol{0}_{r_1}, \!\left\{\!a_{j,0}b_{j,0}\!+\!\sum_{\ell=1}^{r_1}\lambda_{\ell,1}a_{j,0}\sqrt{b_{j,0}b_{\ell,1}} \psi_{j,\ell,0,1}\!\right\}_{\!j=\!1}^{r_0}\!\right]^T
\end{align*}
\begin{align*}
\bE_0&=\begin{bmatrix}
\tilde \bD_0+\bM_0&   \bN_0 \\
 \bN_0^T&  \bD_0
\end{bmatrix}, \   \    \
\bE_1=\begin{bmatrix}
 \bD_1 & \bN_1\\
 \bN_1^T&\tilde \bD_1+\bM_1
\end{bmatrix}, 
\end{align*}
with $\bD_i \in \mathbb{R}^{r_i\times r_i}$, $\tilde \bD_i,\bM_i\in \mathbb{R}^{r_{\tilde i}\times r_{\tilde i}}$ and $\bN_i\in  \mathbb{R}^{r_1\times r_0}$ defined as,
\begin{align*}
\bD_i&=\frac{1}{2}\diag\left\{  (1+\lambda_{j,i}a_{j,i})^2 \right\}_{j=1}^{r_i},\\
\tilde\bD_i&=\diag\left\{ \frac{\phi_{j,\tilde i}^2}{2}+\alpha_{\tilde i}a_{j,\tilde i}b_{j,\tilde i} \right\}_{j=1}^{r_i},\\
 [N_i]_{\ell,j}&= -\frac{1}{2}{(1+\lambda_{\ell,i})^2}a_{\ell,i}a_{j,\tilde i}\psi_{\ell,j,i,\tilde i}^2,\\
  [M_i]_{j,k}&=\alpha_{\tilde i}a_{j,\tilde i}a_{k,\tilde i}\sqrt{b_{j,\tilde i}b_{k,\tilde i}}\sum_{\ell=1}^{r_i}\lambda_{\ell,i}\psi_{\ell,j,i,\tilde i}\psi_{\ell,k,i,\tilde i}.
\end{align*}
}

\end{remark}
Using these deterministic equivalents, a deterministic equivalent of the Fisher ratio $\rho(\bw)$ can be obtained as,
$$
\rho(\bw)-\overline \rho(\bw)\asto 0,
$$
where
$$
\overline \rho(\bw)=\frac{\left|\overline m_0(\bw)- \overline m_1(\bw)\right|}{\sqrt{ \overline v_0(\bw)+ \overline v_1(\bw)}},
$$
Replacing $\overline m_i(\bw)$ and $\overline v_i(\bw)$ by their expressions, our optimization problem can be written as:
\begin{equation}
 \max_\bw \ \  \frac{ \left| \bg^T\bw+\beta_0+\beta_1\right|}{2\sqrt{\bw^T\bE\bw+2\be^T\bw+ b}},
 \label{opt-prob}
\end{equation}
where $\beta_i=\alpha_i+p\left(\frac{\sigma_i^2}{\sigma_{\tilde i}^2}-1\right)$, $\bg=\bg_0-\bg_1$, $\be=\be_0+\be_1$, $\bE=\bE_0+\bE_1$ and $ b=b_0+b_1$.
To simplify the optimization, we perform the change of variable $\tilde \bw= \bE^{\frac{1}{2}}\bw + \bE^{-\frac{1}{2}}\be $.
\begin{proposition} Assume that $\beta_0+\beta_1-\bg^T\bE^{-1}\be\neq 0$. The optimal parameter vector $\bw^\star$ is given by
\begin{equation}
\bw^\star=\bE^{-1}(\theta^\star \bg-\be)
\label{opt_par}
\end{equation}
where $\theta^\star=\frac{b-\be^T\bE^{-1}\be}{|\beta_0+\beta_1-\bg^T\bE^{-1}\be|}$.
\label{optimal_w}
\end{proposition}
\begin{remark}
We assumed in Proposition \ref{optimal_w} that $\beta_0+\beta_1-\bg^T\bE^{-1}\be\neq 0$. Although we did not prove that, it is found to be true in all our extensive simulations on both real and synthetic data.

Until now, we assumed that the constant $\eta$ that appears in the score function of the proposed classifier is known. It should be noted that the optimization of the Fisher ratio is not impacted by this assumption since it does not depend on $\eta$. A possible choice of $\eta$ is the one that ensures equal distance between both means, i.e. $\overline m_0(\bw^\star)+\overline m_1(\bw^\star)=0$.
 The $\eta$ that verifies this equation is:
\begin{equation}
\begin{aligned}
\eta&=-\frac{1}{4}\left[(\bg_0+\bg_1)^T\bw^\star+\alpha_1-\alpha_0+2(c_1-c_0)+p\frac{\sigma_0^4-\sigma_1^4}{\sigma_0^2\sigma_1^2}\right],
\end{aligned}
\label{eta_opt}
\end{equation}
\end{remark}
The optimal design parameters $\bw^\star$ in proposition \ref{optimal_w} could not be directly used in practice, since they depend on the unobservable quantities $\alpha_i$, $\lambda_{j,i}$ and $b_{j,i}$. To solve this issue, consistent estimators for these quantities need to be retrieved. This is the objective of the following result: 
\begin{proposition}Under the settings of Assumption 1, we have
\begin{align*}
&| \lambda_{j,i}-\hat\lambda_{j,i}| \asto 0,\  \ \ \  | \alpha_i-\hat \alpha_i|\asto 0,\\
 & | b_{j,i}-\hat b_{j,i}|\asto 0,  \  \ \ \ \  |  \psi_{\ell,j,1,0} -\hat \psi_{\ell,j,1,0}|\asto 0,
\end{align*}
where
\begin{align*}
&\hat \alpha_i=\frac{\|\hat\bmu\|^2- c_1\sigma_1^2-c_0\sigma_0^2}{\sigma_i^2},\\
&\hat\lambda_{j,i}=\frac{{s_{j,i}/\sigma_i^2}+1-c_i+\sqrt{(s_{j,i}/\sigma_i^2+1-c_i)^2-4s_{j,i}/\sigma_i^2}}{2},\\
&\hat b_{j,i}=\frac{1+c_i/\hat\lambda_{j,i}}{1-c_i/\hat\lambda_{j,i}^2}\frac{\hat\bmu^T \bv_{j,i} \bv_{j,i}^T\hat\bmu}{\|\hat\bmu\|^2- c_1\sigma_1^2-c_0\sigma_0^2}, \\
&\hat \psi_{\ell,j,1,0}= \frac{1}{\sqrt{a_{\ell,1}a_{j,0}}} \bu_{\ell,1}^T\bu_{j,0},
\end{align*}
with $\hat\bmu=\overline {\bf x}_0- \overline {\bf x}_1$ and $s_{j,i}$ is the $j$-th largest eigenvalue of the sample covariance matrix $\hat\bSig_i$ corresponding to class $i$.
\end{proposition}
\begin{proof}
The proof is a direct application of results from \cite{Couillet2011,baik2005} and it is thus omitted. 
\end{proof}

The steps of the design of the proposed classifier are summarized in the following algorithm.

\begin{algorithm}[H]
\label{alg_2}
\caption{Steps for the computation of the proposed classifier decision rule}
\begin{algorithmic}
\State 1. Given the training set corresponding to class $i$, use one of the algorithms of \cite{KRITCHMAN2008,Johnstone2009,Ulfarsson2008,Passemier2017} to estimate $\sigma_i^2$ and $r_i$.
\State 2. Compute $\left\{s_{j,i}\right\}_{j=1}^{r_i}$ the $r_i$ largest eigenvalues of the sample covariance matrix of class $i$ and their corresponding eigenvectors $u_{j,i}$.
\State 3. Compute the parameters of the proposed classifier defined in Theorem \ref{m_v_conv}.
\State 4. Compute $\eta$ using equation \eqref{eta_opt} and the optimal parameter vector $\bw^\star$ using equation \eqref{opt_par}.
\State 5. Plugging $\eta$ and ${\bf w}^\star$ into \eqref{discriminant_function_impQDA} yields the decision rule of the proposed classifier.
\end{algorithmic}
\end{algorithm}

\section{Numerical Simulations}
\label{Numerical_Simulations}
In this section, we compare the performance of the proposed improved QDA classifier with R-QDA classifier using both synthetic and real data.
\subsection{Synthetic data}
For the synthetic data simulations, we used the following protocol for Montecarlo estimation of the true misclassification rate:
\begin{itemize}
\item Step 1: Set $r_0=r_1=3$, orthogonal symmetry breaking directions as follows: 
\begin{align*}
&\bV_0=[\bv_{1,0},\bv_{2,0},\bv_{3,0}]=[\bI_{3\times3},\boldsymbol{0}_{3,p-3}]^T\\
&\bV_1=[\bv_{1,1},\bv_{2,1},\bv_{3,1}]=[\boldsymbol{0}_{3\times3},\bI_{3\times3},\boldsymbol{0}_{3,p-6}]^T
\end{align*}
and their corresponding weights $\lambda_{1,0}=5$, $\lambda_{2,0}=4$, $\lambda_{3,0}=3$, $\lambda_{1,1}=6$, $\lambda_{2,1}=5$, $\lambda_{3,1}=4$. Set $\bmu_0=\frac{a}{\sqrt{p}}[1,1,\cdots,1]^T$ and $\bmu_1=-\bmu_0$ where $a$ is a finite constant. In our simulations, we choose $a=0.5$ and $a= 0.8$.
\item Step 2: Generate $n_i$ training samples for class $i$.
\item Step 3: Using the training set, design the improved QDA classifier as explained in section \ref{Improved_QDA}.
\item Step 4: Estimate the true misclassification rate of both classifiers using a set of 2000 testing samples. For the R-QDA classifier, a grid search over $\gamma \in \{ 10^{i/10}, i=-10:1:10\}$ is performed.
\item Step 5: Repeat Step 2--4, 250 times and determine the average misclassification rate of both classifiers.
\end{itemize}
In \figref{fig_1}, we plot the misclassification rate vs. training sample size $n$ when $p=500$, $\sigma_0^2=\sigma_1^2=1$ and $\pi_0=\pi_1=1/2$ for the proposed improved QDA and the classical R-QDA using synthetic data. It is observed that the improved QDA outperforms the classical R-QDA and the gap between the two schemes is significant.
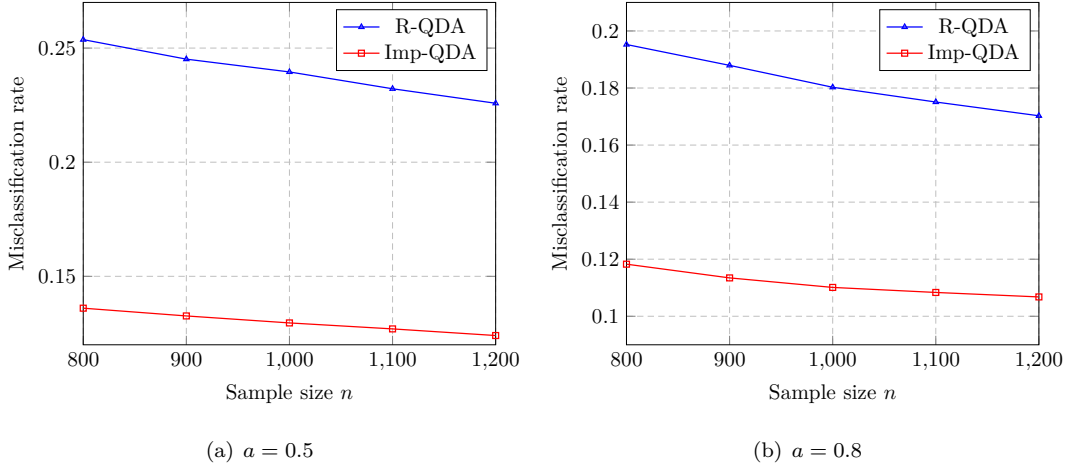
\begin{figure}
  \centering
  \subfigure[$a=0.5$]  
{  
   \begin{tikzpicture}[scale=0.8,font=\normalsize]
    \renewcommand{\axisdefaulttryminticks}{4}
    \pgfplotsset{every major grid/.style={densely dashed}}
    \tikzstyle{every axis y label}+=[yshift=-20pt]
    \tikzstyle{every axis x label}+=[yshift=5pt]
    legend style={fill=white,        at={(0.98,0.98)}, anchor=north east, font=\normalsize}
  
    \begin{axis}[
      xmin=800,
      xmax=1200,
      ymin=0.12,
      ymax=0.27,
      grid=major,
      scaled ticks=true,
   			xlabel={Sample size $n$},
   			ylabel={Misclassification rate}			
      ]
   \addplot[color=blue,mark size=1.4pt,mark=triangle,line width=0.6pt,error bars/.cd,y dir=both,y explicit, error bar style={mark size=1.5pt}]  plot coordinates{    
  (800,2.537040e-01)(900,2.451500e-01)(1000,2.395460e-01)(1100,2.321720e-01)(1200,2.258200e-01)};\addlegendentry{R-QDA}           

\addplot[color=red,mark size=1.4pt,mark =square,line width=0.6pt,error bars/.cd,y dir=both,y explicit, error bar style={mark size=1.5pt}]  plot coordinates{ 
                  (800,1.360360e-01)(900,1.326380e-01)(1000,1.295780e-01)(1100,1.269600e-01)(1200,1.240560e-01)
};\addlegendentry{ Imp-QDA }
           \end{axis}
  \end{tikzpicture}}
  \subfigure[$a=0.8$]  
{  \begin{tikzpicture}[scale=0.8,font=\normalsize]
    \renewcommand{\axisdefaulttryminticks}{4}
    \pgfplotsset{every major grid/.style={densely dashed}}
    \tikzstyle{every axis y label}+=[yshift=-20pt]
    \tikzstyle{every axis x label}+=[yshift=5pt]
    legend style={fill=white,        at={(0.98,0.98)}, anchor=north east, font=\normalsize}
    \begin{axis}[
      xmin=800,
      ymin=0.09,
      xmax=1200,
      ymax=0.21,
      grid=major,
      scaled ticks=true,
   			xlabel={Sample size $n$},
   			ylabel={Misclassification rate}			
      ]
 \addplot[color=blue,mark size=1.4pt,mark=triangle,line width=0.6pt,error bars/.cd,y dir=both,y explicit, error bar style={mark size=1.5pt}]  plot coordinates{  (800,1.952500e-01)(900,1.879360e-01)(1000,1.802640e-01)(1100,1.750900e-01)(1200,1.702720e-01)
 		};
\addlegendentry{R-QDA}                  
  \addplot[color=red,mark size=1.4pt,mark =square,line width=0.6pt,error bars/.cd,y dir=both,y explicit, error bar style={mark size=1.5pt}]  plot coordinates{ 
(800,1.182540e-01)(900,1.134460e-01)(1000,1.101060e-01)(1100,1.083420e-01)(1200,1.067740e-01)	};
\addlegendentry{ Imp-QDA }
           \end{axis}
  \end{tikzpicture} }

\centering
  \caption{Misclassification rate vs. sample size $n$ for $p=500$, $\sigma_0^2=\sigma_1^2=1$ and $\pi_0=\pi_1=1/2$. Comparison between Improved QDA and R-QDA with synthetic data. } 
  \label{fig_1}
\end{figure}
 \begin{table}[H]
   \caption{Misclassification error for $n=1000$, $p=500$, $a=0.5$, $\sigma_0^2=1$ and different values of $\sigma_1^2$.}
 \begin{center}
    \begin{tabular}{ |c|c|c|c| c|}
    \hline
   & $\sigma_1^2=1.2$ &  $\sigma_1^2=1.5$  & $\sigma_1^2=2$\\ 
 \hline
 R-QDA   & $0.205_{0.008}$ & $0.102_{0.007}$ & $0.0133_{0.002}$\\  
 \hline
 Imp-QDA   & $ 0.097_{0.007}$ & $0.001_{0.0008}$ & $0.000_{0.000} $\\
 \hline
  \end{tabular}
    \end{center} 
        \label{table1}
      \end{table}

As a second investigation, we study the impact of the difference between the noise variances $\sigma_0^2$ and $\sigma_1^2$. Table \ref{table1} reports the misclassification rate of the R-QDA classifier and our proposed classifier for fixed $\sigma_0^2$ and different values of $\sigma_1^2$. As can be seen, the improved QDA outperforms the classical R-QDA and exploits better the difference between $\sigma_0^2$ and $\sigma_1^2$. Such a finding is expected since as the difference $|\sigma_0^2-\sigma_1^2|$ increases, the classes become more distinguishable, resulting in better performances. The R-QDA is not able to leverage this difference well since it undergoes a higher estimation error in the covariance matrix, which affects its performance considerably.

\subsection{Real data}

For real data simulation, we use two datasets. The first one is the epileptic seizure detection dataset, which consists of recordings of brain activity using EEG signals. The dataset is composed of 5 classes with 2300 samples of dimension $p=178$ available for each class.  In our simulation, we consider the most confusing classes of this dataset for binary classification, namely class 4, which corresponds to recordings where the patients had their eyes closed and class 5, which corresponds recordings where the patients had their eyes open. This dataset is publicly available at \url{https://archive.ics.uci.edu/ml/datasets/Epileptic+Seizure+Recognition}.

The second dataset considered in this paper is the Gisette dataset composed of handwritten digits. The objective is to separate the highly confusing digits '4' and '9'. In our simulation, prior to applying the classification technique, a standard PCA is applied in order to reduce the observation size. This is a standard procedure in machine learning and is referred to as feature selection. We leverage all the data available in the training and validation data sets. A subset of these samples serves to build the classifier, while the remaining samples are used as a test data set to estimate the misclassification rate.    
 This dataset is publicly available at \url{https://archive.ics.uci.edu/ml/datasets/Gisette}. We used the following protocol for the real dataset:
\begin{itemize}
\item Step 1: Let $q_0$ be the ratio between the total number of samples in class $\mathcal{C}_0$ to the total number of samples available in the full dataset. Denote by $n_{\rm Full}$ the total number of samples in the full dataset. Choose $n<n_{\rm Full}$ the number of training samples; set $n_0=\lfloor q_0 n \rfloor$, where $\lfloor . \rfloor$ is the floor function and $n_1=n-n_0$.
Take $n_i$ training samples belonging to class $\mathcal{C}_i$ randomly from the full dataset. The remaining samples are used as a test dataset in order to estimate the classification error.
\item Step 2: Using the training dataset, design the improved QDA classifier, as explained in section \ref{Improved_QDA}.
\item Step 3: Using the test dataset, estimate the true classification error for both classifiers. For the R-QDA classifier, a grid search over $\gamma \in \{ 10^{i/10}, i=-10:1:10\}$ is performed.
\item Step 4: Repeat steps 1--4, 250 times, and determine the average misclassification rate of both classifiers.
\end{itemize}

\begin{figure}
  \centering
  
  { \begin{tikzpicture}[scale=0.8,font=\small]
    \renewcommand{\axisdefaulttryminticks}{4}
    \pgfplotsset{every major grid/.style={densely dashed}}
    \tikzstyle{every axis y label}+=[yshift=-20pt]
    \tikzstyle{every axis x label}+=[yshift=5pt]
    legend style={fill=white,        at={(0.98,0.98)}, anchor=north east, font=\normalsize}
    \begin{axis}[
      xmin=500,
      xmax=2500,
      ymin=0.25,
      ymax=0.35,
      grid=major,
      scaled ticks=true,
   			xlabel={Sample size $n$},
   			ylabel={Misclassification rate}			
      ]
   \addplot[color=blue,mark size=1.4pt,mark=triangle,line width=0.6pt,error bars/.cd,y dir=both,y explicit, error bar style={mark size=1.5pt}]  plot coordinates{    
  (500,3.298244e-01)(1000,3.287694e-01)(1500,3.280097e-01)(2000,3.270962e-01)(2500,3.2640e-01)};\addlegendentry{R-QDA}           

\addplot[color=red,mark size=1.4pt,mark =square,line width=0.6pt,error bars/.cd,y dir=both,y explicit, error bar style={mark size=1.5pt}]  plot coordinates{ 
                  (500,2.680146e-01)(1000,2.678722e-01)(1500,2.6757677e-01)(2000,2.672577e-01)(2500,2.67010052e-01)
};\addlegendentry{ Imp-QDA }
           \end{axis}
  \end{tikzpicture}}
 
\centering
  \caption{Misclassification rate vs. sample size $n$ for $p=98$. Comparison between Improved QDA and R-QDA with elliptic seizure detection dataset. } 
  \label{fig_EEG}
  \end{figure}
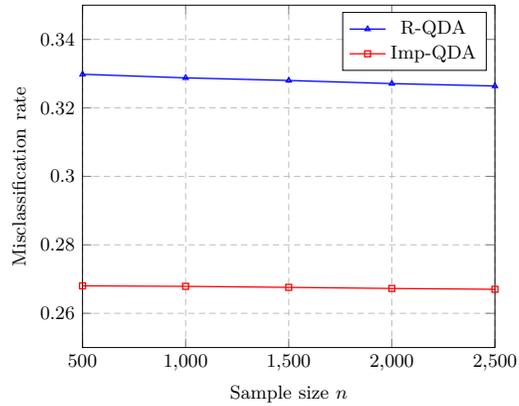
  
In \figref{fig_EEG}, we compare the performance of the proposed classifier with that of the R-QDA classifier when used for the elliptic seizure detection dataset. The misclassification rate of both classifiers is plotted versus the number of training samples. As observed, the proposed classifier outperforms the classical R-QDA significantly.

In \figref{fig_gisette}, the performance of the proposed classifier is assessed along with that of the classical R-QDA when the  Gisette dataset is considered. We note the important gain of the proposed Imp-QDA similarly.


\begin{figure}
  \centering
  
  { \begin{tikzpicture}[scale=0.8,font=\small]
     \renewcommand{\axisdefaulttryminticks}{4}
    \pgfplotsset{every major grid/.style={densely dashed}}
    \tikzstyle{every axis y label}+=[yshift=-20pt]
    \tikzstyle{every axis x label}+=[yshift=5pt]
    legend style={fill=white,        at={(0.98,0.98)}, anchor=north east, font=\normalsize}
    \begin{axis}[
      xmin=400,
      xmax=700,
      ymin=0.04,
      ymax=0.09,
      grid=major,
      scaled ticks=true,
   			xlabel={Sample size $n$},
   			ylabel={Misclassification rate}			
      ]
   \addplot[color=blue,mark size=1.4pt,mark=triangle,line width=0.6pt,error bars/.cd,y dir=both,y explicit, error bar style={mark size=1.5pt}]  plot coordinates{    
 (400,8.027121e-02)(500,7.453385e-02)(600,7.085781e-02)(700,7.065714e-02)};\addlegendentry{R-QDA}           

\addplot[color=red,mark size=1.4pt,mark =square,line width=0.6pt,error bars/.cd,y dir=both,y explicit, error bar style={mark size=1.5pt}]  plot coordinates{ 
                 (400,6.267273e-02)(500,5.311231e-02)(600,4.802031e-02)(700,4.623175e-02)

};\addlegendentry{ Imp-QDA }
           \end{axis}
  \end{tikzpicture}}
 
\centering
  \caption{Misclassification rate vs. sample size $n$ for $p=98$. Comparison between Improved QDA and R-QDA with gisette dataset. } 
  \label{fig_gisette}
\end{figure}
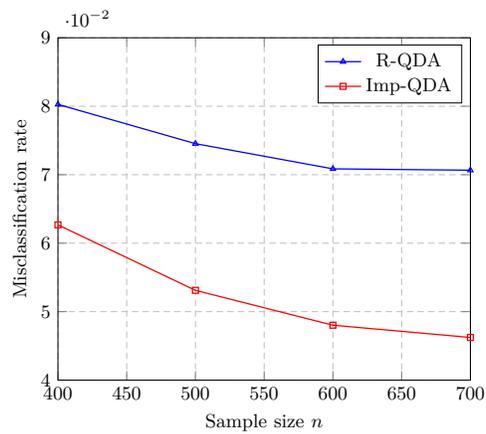

\begin{table}[H]
 \caption{Misclassification rate for the binary classification of class 4 and class 5 of the elliptic seizure detection dataset. Comparaison between the prposed classifier and classical techniques.}
 \begin{center}
    \begin{tabular}{ |c|c|c|c| c|}
    \hline
   & $n=200$ &  $n=1000$  & $n=2000$\\ 
 \hline
  Imp-QDA  & ${\bf 0.270}$ & ${\bf 0.268}$ & ${\bf 0.267}$\\  
 \hline
  R-QDA   & $ 0.337$ & $ 0.328$ & $0.327$\\
 \hline
  SVM (lin) & $ 0.485$ & $0.479$ & $0.474$\\
 \hline
  SVM (Poly$_3$) & $ 0.389$ & $0.299$ & $0.270$\\
 \hline
  KNN$_{1}$   & $ 0.395$ & $0.354$ & $0.335$\\
 \hline
  KNN$_{5}$   & $ 0.432$ & $0.388$ & $0.369$\\
 \hline
  \end{tabular}
    \end{center} 
      \label{table2}
      \end{table}
      
      As a final investigation, using the elliptic seizure dataset, we compare the performance of the proposed classifier with other standard classifiers such as support vector machine (SVM) and k-nearest neighbors (KNN). For SVM, linear and polynomial kernels are used, and for KNN, the number of neighbors used is 1 and 5.
 The Imp-QDA outperforms all these classifiers. Moreover, a larger training set is needed for these classifiers to approach the performance of Imp-QDA. For instance, polynomial SVM requires a training set of size $n=2000$ to achieve the performance of our classifier with a training set of size $n=200$.

\section{Conclusion}
\label{conclusion}
In this paper, we proposed an improved QDA classifier that is shown to outperform the classical R-QDA while requiring lower computation complexity. The proposed classifier is more suited for spiked covariance populations; a situation frequently met in EEG signal processing, detection, and econometrics applications.
The obtained results are very promising, opening the path to extend the analysis to more general covariance models such as a diagonal-plus-low-rank-perturbation model.

\appendix
\section{Proof of Proposition 3}
\label{proof_dist_conv}
Replacing $\hat \bC_0^{-1}$ and $\hat \bC_1^{-1}$ by their expressions, one can easily get
\begin{align*}
 Y_i(\hat\bC_0,\hat\bC_1)&=\left(\frac{1}{\sigma_1^2}-\frac{1}{\sigma_0^2}\right)\bz^T \bSig_i\bz+\nu_i+2\by_i^T\bz-\xi_i
\end{align*}
where $$\nu_i= \frac{1}{\sigma_1^2}\sum_{j=1}^{r_1} w_{j,1}( \tilde\bz_i^T\bu_{j,1})^2 -\frac{1}{\sigma_0^2} \sum_{j=1}^{r_0} w_{j,0} ( \tilde\bz_i^T  \bu_{j,0})^2.$$ Applying the trace lemma \cite{Couillet2011}, we have
$$
\frac{1}{p}\bz^T \bSig_i\bz-\frac{1}{p}\tr(\bSig_i)\asto 0,
$$
The assumed spiked model implies that $\frac{1}{p}\tr(\bSig_i)\longrightarrow \sigma_i^2$. Thus, 
$$
\left(\frac{1}{\sigma_1^2}-\frac{1}{\sigma_0^2}\right)\bz^T \bSig_i\bz-p\sigma_i^2\left(\frac{1}{\sigma_1^2}-\frac{1}{\sigma_0^2}\right)\asto 0,
$$
Using Slutsky's theorem, we can conclude that
$$
 Y_i(\hat\bC_0,\hat\bC_1) -\left( p\sigma_i^2\left(\frac{1}{\sigma_1^2}-\frac{1}{\sigma_0^2}\right)+\nu_i+2\by_i^T\bz-\xi_i\right)\dto 0$$
which concludes the proof.

\section{Proof of Theorem 4}
\label{proof_m_v_conv}
First, we recall the following results from \cite{Couillet2011} that will be used throughout the proof:
\begin{equation}
\begin{aligned}
&\bv_{j,i}^T\bu_{k,i}\bu_{k,i}^T\bv_{j,i}- a_{j,i}\delta_{j,k} \asto 0,\\ 
&\bv_{j,i}^T\bu_{k,\ell}\bu_{k,\ell}^T\bv_{j,i}- a_{k,\ell}(\bv_{j,i}^T\bv_{k,\ell})^2 \asto 0,\\ 
&\frac{1}{\|\bmu\|^2}\bmu^T\bu_{j,i}\bu_{j,i}^T\bmu- a_{j,i} b_{j,i} \asto 0,
\end{aligned}
\label{conver_a_b}
\end{equation}
where $\delta_{j,k}$ is Kronecker delta.
We shall also recall the following formula allowing to compute the variance and covariance of  quadratic forms of a multivariate normal distribution. If ${\bf z}\sim \mathcal{N}({\bf 0},{\bf I}_p)$ and ${\bf Q}$ is a deterministic $p\times p$ matrix, then: 
\begin{equation}
{\rm var}({\bf z}^{T}{\bf Q}{\bf z})=2\tr{\bf Q}^2
\label{eq:var}
\end{equation}
Let ${\bf Q}_1$ and ${\bf Q}_2$ be two deterministic $p\times p$ matrices, we have similarly:
\begin{equation}
{\rm cov}\left({\bf z}^{T}{\bf Q}_1{\bf z},{\bf z}^{T}{\bf Q}_2{\bf z}\right)=2\tr {\bf Q}_1{\bf Q}_2
\label{eq:cov}
\end{equation}
The mean of $\tilde Y_i$ is given by,
\begin{align*}
m_i(\bw)=p\sigma_i^2\left(\frac{1}{\sigma_1^2}-\frac{1}{\sigma_0^2}\right)+\tilde\nu_i-\xi_i,
\end{align*}
where $\tilde\nu_i=\sum_{j=1}^{r_1} \frac{w_{j,1}}{\sigma_1^2} \bu_{j,1}^T \bSig_i \bu_{j,1}-
\sum_{j=1}^{r_0}\frac{w_{j,0}}{\sigma_0^2}\bu_{j,0}^T \bSig_i\bu_{j,0}$.
Let us begin by treating the term $\xi_i$. First, we have
\begin{align*}
\xi_i&=-2\eta^{\rm Imp-QDA}+ (\bmu_i-\overline\bx_0)^T\hat\bC_0^{-1}(\bmu_i-\overline\bx_0)-(\bmu_i-\overline\bx_1)^T\hat\bC_1^{-1}(\bmu_i-\overline\bx_1),
\end{align*}
Noting that $\overline\bx_i=\bmu_i+\frac{1}{n_i}\bOmega_i\boldsymbol{1}_{n_i}$ where $\bOmega_i=\bSig_i^{\frac{1}{2}}\bZ_i$ and $\bZ_i \in \mathbb{R}^{p\times n_i}$ with entries i.i.d. $\mathcal{N}(0,1)$, we can write
\begin{align}
\xi_0&=-2\eta^{\rm Imp-QDA}+ \frac{1}{n_0^2}\boldsymbol{1}_{n_0}^T\bOmega_0^T\hat\bC_0^{-1}\bOmega_0\boldsymbol{1}_{n_0}-(\bmu-\frac{1}{n_1}\bOmega_1\boldsymbol{1}_{n_1})^T\hat\bC_1^{-1}(\bmu-\frac{1}{n_1}\bOmega_1\boldsymbol{1}_{n_1}),
\label{xi_worked} 
\end{align}
Let $\bar\bz_i=\frac{1}{\sqrt{n_i}} \bOmega_i\boldsymbol{1}_{n_i}$. The sample covariance matrix $\hat\bSig_i$ is independent of $\bar\bz_i$ \cite{Zollanvari2015}, which means also that $\bar\bz_i$ is independent of the eigenvectors of $\hat\bSig_i$ that appears in $\hat\bC_i^{-1}$. Thus, we have
\begin{align}
\frac{1}{n_i}\bmu^T\hat\bC_i^{-1}\bOmega_i\boldsymbol{1}_{n_i}\asto 0,\label{conv0}\end{align}
\begin{align}
\frac{1}{n_i^2}\boldsymbol{1}_{n_i}^T\bOmega_i^T\hat\bC_i^{-1}\bOmega_i\boldsymbol{1}_{n_i}-\frac{1}{n_i}\tr \bSig_i\hat\bC_i^{-1}\asto 0,\label{conv11}
\end{align}
Replacing $\bSig_i$ and $\hat\bC_i^{-1}$ by their expressions in \eqref{conv11} and using the fact that $r$ is finite, we have $\frac{1}{n_i}\tr \bSig_i\hat\bC_i^{-1}-{c_i}\asto 0$. Thus,
\begin{align}
\frac{1}{n_i^2}\boldsymbol{1}_{n_i}^T\bOmega_i^T\hat\bC_i^{-1}\bOmega_i\boldsymbol{1}_{n_i}-c_i\asto 0,\label{conv1}
\end{align}
On the other hand, replacing $\hat\bC_i^{-1}$ by its expression and applying \eqref{conver_a_b}, one can get easily
 \begin{equation}
\begin{aligned}
\bmu^T\hat\bC_i^{-1}\bmu - \frac{\|\bmu\|^2}{\sigma_i^2}\left(1+ \sum_{j=1}^{r_i} w_{j,i} a_{j,i}b_{j,i}\right)\asto 0.
  \end{aligned}
  \label{conv2}
\end{equation}
Combining \eqref{xi_worked}, \eqref{conv0}, \eqref{conv1} and \eqref{conv2}, we get
\begin{equation}
\xi_0 - \left[-2\eta+c_0-c_1- \frac{\|\bmu\|^2}{\sigma_1^2}\left(1+ \sum_{j=1}^{r_1} w_{j,1} a_{j,1}b_{j,1}\right)\right]\asto 0
\label{xi_0_conv}
\end{equation}
Applying the same approach, one can prove that
\begin{equation}
\xi_1 - \left[-2\eta+c_0-c_1+ \frac{\|\bmu\|^2}{\sigma_0^2} \left(1+\sum_{j=1}^{r_0} w_{j,0} a_{j,0}b_{j,0}\right)\right]\asto 0
\label{xi_1_conv}
\end{equation}
Moreover, replacing $\bSig_0$ and $\bSig_1$ by their expressions and applying \eqref{conver_a_b}, we have
\begin{equation}
\begin{aligned}
&\nu_0 - \left[ \frac{\sigma_0^2}{\sigma_1^2}\sum_{j=1}^{r_1} w_{j,1} \phi_{j,1}- \sum_{j=1}^{r_0} w_{j,0} (1+a_{j,0}\lambda_{j,0})\right]\asto 0\\
&\nu_1 - \left[ \sum_{j=1}^{r_1} w_{j,1} (1+a_{j,1}\lambda_{j,1})- \frac{\sigma_1^2}{\sigma_0^2}\sum_{j=1}^{r_0} w_{j,0} \phi_{j,0}\right]\asto 0
\end{aligned}
\label{nu_i_conv}
\end{equation}
Combining \eqref{xi_0_conv}, \eqref{xi_1_conv} and \eqref{nu_i_conv}, we obtain the first convergence result of Theorem 4.
Now, we address the convergence of $v_i(\bw)$. We will treat the term $v_1(\bw)$ only. The convergence of $v_0(\bw)$ can be obtained by applying the same steps.
Since $\bz$ is Gaussian, it is not hard to see that $\nu_1$ and $\by_1^T\bz$ are uncorrelated. Thus, we have
$$
v_1(\bw)= \var(\nu_1)+4 \var(\by_1^T\bz),
$$
Let us begin by $\var(\nu_i)$ which can be written as
\begin{align*}
 \var(\nu_1)&=\frac{1}{\sigma_1^4}\sum_{j=1}^{r_1}w_{j,1}^2 \var\left( \left[\bu_{j,1}^T\bSig_1^{\frac{1}{2}}\bz\right]^2\right)+\frac{1}{\sigma_0^4}\sum_{j=1}^{r_0}w_{j,0}^2\var\left( \left[\bu_{j,0}^T\bSig_1^{\frac{1}{2}}\bz\right]^2\right)\\
 &-\sum_{\ell=1}^{r_1}\sum_{j=1}^{r_0}\frac{2w_{\ell,1}w_{j,0}}{\sigma_1^2\sigma_0^2}{\rm cov}\!\left(\!\left[\bu_{\ell,1}^T\bSig_1^{\frac{1}{2}}\bz\right]^2\!,\left[\bu_{j,0}^T\bSig_1^{\frac{1}{2}}\bz\right]^2\!\right)
\end{align*}
where we have used in the last equation the fact that $\bu_{j,k}^T\bSig_1^{\frac{1}{2}}\bz$ is independent of $\bu_{j',k}^T\bSig_1^{\frac{1}{2}}\bz$ for $j'\neq j,\ k=0,1$, a fact that follows from  the orthogonality between eigenvectors $\bu_{j,k}$ and $\bu_{j',k}$.

Using \eqref{eq:var}, we obtain
\begin{align*}
\var\left( \left[\bu_{j,1}^T\bSig_1^{\frac{1}{2}}\bz\right]^2\right)&=2\tr \left[ \bSig_1^{\frac{1}{2}} \bu_{j,1}\bu_{j,1}^T\bSig_1^{\frac{1}{2}}\right]^2=2 \left[\bu_{j,1}^T\bSig_1\bu_{j,1}\right]^2
\end{align*}
Replacing $\bSig_1$ by its expression and applying \eqref{conver_a_b}, we can easily show that
$$
\left[\bu_{j,1}^T\bSig_i\bu_{j,1}\right]^2-\sigma_1^4(1+\lambda_{j,1}a_{j,1})^2\asto 0.
$$
Thus, we have 
\begin{equation}
\var\left( \left[\bu_{j,1}^T\bSig_1^{\frac{1}{2}}\bz\right]^2\right)-2\sigma_1^4(1+\lambda_{j,1}a_{j,1})^2\asto 0.
\label{var1}
\end{equation}
Similarly, we have
\begin{align*}
\var\left( \left[\bu_{j,0}^T\bSig_1^{\frac{1}{2}}\bz\right]^2\right)=2 \left[\bu_{j,0}^T\bSig_1\bu_{j,0}\right]^2
\end{align*}
Applying \eqref{conver_a_b} again, we can easily show that
$$
\left[\bu_{j,0}^T\bSig_1\bu_{j,0}\right]^2-\sigma_1^4\phi_{j,0}^2\asto 0.
$$
Thus, we have 
\begin{equation}
\var\left( \left[\bu_{j,0}^T\bSig_1^{\frac{1}{2}}\bz\right]^2\right)-2\sigma_1^4\phi_{j,0}^2\asto 0.
\label{var2}
\end{equation}
Using now \eqref{eq:cov}, we obtain:
\begin{align*}
{\rm cov}\left(\left[\bu_{\ell,1}^T\bSig_1^{\frac{1}{2}}\bz\right]^2,\left[\bu_{j,0}^T\bSig_1^{\frac{1}{2}}\bz\right]^2\right)=2 \left[\bu_{j,0}^T\bSig_1\bu_{\ell,1}\right]^2 \end{align*}
Applying \eqref{conver_a_b} again, we obtain
\begin{equation}
\left[\bu_{j,0}^T\bSig_1\bu_{\ell,1}\right]^2-\sigma_1^4a_{\ell,1}a_{j,0}(1+\lambda_{\ell,1})^2(\bv_{j,0}^T\bv_{\ell,1})^2\asto 0.
\label{covv}
\end{equation}
Combining \eqref{var1}, \eqref{var2} and \eqref{covv}, we obtain
\begin{equation}
 \var(\nu_1)-\overline v_{1,1}\asto 0,
 \label{var_nui}
\end{equation}
where
\begin{align*}
\overline v_{1,1}&=2\sum_{j=1}^{r_1}w_{j,1}^2 (1+\lambda_{j,1}a_{j,1})^2+2\frac{\sigma_1^4}{\sigma_0^4}\sum_{j=1}^{r_0}w_{j,0}^2\phi_{j,0}^2
-4\frac{\sigma_1^2}{\sigma_0^2}\sum_{\ell=1}^{r_1}\sum_{j=1}^{r_0}w_{\ell,1}w_{j,0}a_{\ell,1}a_{j,0}(1+\lambda_{\ell,1})^2(\bv_{j,0}^T\bv_{\ell,1})^2. \end{align*}
 It remains now to deal with the term $\var(\by_1^T\bz)$, which can be written as
\begin{align*}
\var(\by_1^T\bz)&=\mathbb{E} \by_1^T\bz\bz^T\by_1= \by_1^T\by_1\\
&= \left(-\hat\bC_1^{-1}\frac{\bOmega_1\boldsymbol{1}_{n_1}}{n_1}+ \hat\bC_0^{-1}(\bmu+\frac{\bOmega_0\boldsymbol{1}_{n_0}}{n_0})\right)^T\bSig_1\left(-\hat\bC_1^{-1}\frac{\bOmega_1\boldsymbol{1}_{n_1}}{n_1}+ \hat\bC_0^{-1}(\bmu+\frac{\bOmega_0\boldsymbol{1}_{n_0}}{n_0})\right)
 \end{align*}
Using the same arguments as in \eqref{conv0}, the following convergence holds
\begin{align*}
\frac{1}{n_i}\bmu^T\hat\bC_i^{-1}\bSig_1\hat\bC_i^{-1}\bOmega_i\boldsymbol{1}_{n_i}\asto 0,\end{align*}
The independence of $\bOmega_1$ and $\bOmega_0$ yields
\begin{align*}
\frac{1}{n_0n_1}\boldsymbol{1}_{n_0}^T\bOmega_0^T\hat\bC_0^{-1}\bSig_1\hat\bC_1^{-1}\bOmega_1\boldsymbol{1}_{n_1}\asto 0,\end{align*}
while the trace lemma \cite[Theorem 3.4]{Couillet2011} yields:
\begin{align*}
\frac{1}{n_i^2}\boldsymbol{1}_{n_i}^T\bOmega_i^T\hat\bC_i^{-1}\bSig_1\hat\bC_i^{-1}\bOmega_i\boldsymbol{1}_{n_i}-\frac{1}{n_i}\tr \bSig_1\hat\bC_i^{-1}\bSig_1\hat\bC_i^{-1} \asto 0,\end{align*}
Replacing $\bSig_i$ and $\hat\bC_i^{-1}$ by their expressions using the fact that $r$ is finite, we have $\frac{1}{n_1}\tr \bSig_1\hat\bC_1^{-1}\bSig_1\hat\bC_1^{-1}-{c_1}\asto 0$ and $\frac{1}{n_0}\tr \bSig_1\hat\bC_0^{-1}\bSig_1\hat\bC_0^{-1}-{c_0}\frac{\sigma_1^4}{\sigma_0^4}\asto 0$ .
Thus, 
\begin{align*}
\frac{1}{n_i^2}\boldsymbol{1}_{n_i}^T\bOmega_i^T\hat\bC_i^{-1}\bSig_1\hat\bC_i^{-1}\bOmega_i\boldsymbol{1}_{n_i}-c_i\frac{\sigma_1^4}{\sigma_i^4} \asto 0,\end{align*}
It remains to deal with the term $\bmu^T\hat\bC_0^{-1}\bSig_1\hat\bC_0^{-1}\bmu$.
Applying \eqref{conver_a_b}, one can obtain after standard calculations:
\begin{align}
\bmu^T\hat\bC_0^{-1}\bSig_1\hat\bC_0^{-1}\bmu-\overline v_{0,1}\asto 0
\end{align}
where $\overline v_{0,1}$ is given by
{\small
\begin{align*}
\overline v_{0,1}&=\frac{\|\bmu\|^2\sigma_1^2}{\sigma_0^4}\left[1+\sum_{\ell=1}^{r_1}\lambda_{\ell,1}b_{\ell,1}+2\sum_{j=1}^{r_0}w_{j,0}a_{j,0}b_{j,0}+2\!\sum_{j=1}^{r_0}\sum_{\ell=1}^{r_1}\!w_{j,0}\lambda_{\ell,1}a_{j,0}\!\sqrt{b_{j,0}b_{\ell,1}}\bv_{j,0}^T\bv_{\ell,1}\!+\!\sum_{j=1}^{r_0}w_{j,0}^2a_{j,0}b_{j,0}\right.\\&\left. +\!\sum_{k,j=1}^{r_0}\sum_{\ell=1}^{r_1}w_{j,0}w_{k,0}\lambda_{\ell,1}a_{j,0}a_{k,0}\sqrt{b_{j,0}b_{k,0}}\bv_{j,0}^T\bv_{\ell,1}\bv_{k,0}^T\bv_{\ell,1}
\!\right]
\end{align*}}
Putting all these results together and writing the result in vector form yields the convergence of the variance $v_1(\bw)$.

\section{Proof of Proposition 5}
\label{proof_optimal_w}
Using the change of variables $\tilde \bw=  \bE^{\frac{1}{2}}\bw+\bE^{-\frac{1}{2}}\be$, our optimization problem can be written as,
\begin{equation}
 \max_{\tilde \bw} \ \  \left|f(\tilde \bw)\right|,
\end{equation}
where
$$
f(\tilde \bw)=\frac{ \bg^T \bE^{-\frac{1}{2}} \tilde\bw+d}{2\sqrt{\|\tilde\bw\|^2+ b-\be^T\bE^{-1}\be}}
$$
with $d= \beta_0+\beta_1- \bg^T\bE^{-1}\be$.
If at optimality we have $f(\tilde\bw^\star)<0$, then $\max_{\tilde \bw} \left|f(\tilde \bw)\right|=\max_{\tilde \bw}  -f(\tilde \bw)$. Moroever, if $f(\tilde\bw^\star)\geq0$, then $\max_{\tilde \bw} \left|f(\tilde \bw)\right|=\max_{\tilde \bw}  f(\tilde \bw)$. Clearly, we can conclude that $$\left|f(\tilde\bw^\star)\right|=\max\left\{\max_{\tilde \bw}  -f(\tilde \bw),\max_{\tilde \bw}  f(\tilde \bw)\right\}$$
It remains now to solve these two problems $\mathcal{P}_1:\max_{\tilde \bw}  f(\tilde \bw)$ and $\mathcal{P}_2:\max_{\tilde \bw}  -f(\tilde \bw)$.
Let us begin by solving $\mathcal{P}_1:\max_{\tilde \bw}  f(\tilde \bw)$, which can be reformulated, by separating the optimization over the norm and the direction of $\tilde\bw$, as
\begin{equation}
\max_{\theta_1} \  \max_{\|\bar \bw_1\|=1} \ \  \frac{ \theta_1\bg^T \bE^{-\frac{1}{2}} \bar\bw_1+d}{2\sqrt{\theta_1^2+ b-\be^T\bE^{-1}\be}},
\end{equation}
Clearly, the optimal direction is $\bar \bw_1^\star=\frac{ \bE^{-\frac{1}{2}} \bg}{\| \bE^{-\frac{1}{2}} \bg \|}$, thus it remains to solve the following problem
 \begin{equation}
\max_{\theta_1\geq 0}  \ \  \frac{ \theta_1\sqrt{\bg^T \bE^{-1} \bg}+d}{2\sqrt{\theta_1^2+ b-\be^T\bE^{-1}\be}},
\end{equation}
If $d>0$, function $\theta\mapsto \frac{ \theta_1\sqrt{\bg^T \bE^{-1} \bg}+d}{2\sqrt{\theta_1^2+ b-\be^T\bE^{-1}\be}}$ is maximized when $\theta=\theta_1^\star$ with  
$$
\theta_1^\star=\frac{\sqrt{\bg^T \bE^{-1} \bg}(b-\be^T\bE^{-1}\be)}{d} 
$$
On the other hand, if $d<0$, $\theta\mapsto \frac{ \theta_1\sqrt{\bg^T \bE^{-1} \bg}+d}{2\sqrt{\theta_1^2+ b-\be^T\bE^{-1}\be}}$ is strictly increasing and tends to $\frac{1}{2}  \sqrt{\bg^T \bE^{-1} \bg}$ when $\theta\to\infty$. 
We thus conclude 
$$
\sup_{\tilde{\bf w}} f(\tilde{\bf w}) = \begin{cases} \ \ \frac{1}{2}\sqrt{{\bf g}^{T}{\bf E}^{-1}{\bf g}+\frac{d^2}{b-{\bf e}^{T}{\bf E}^{-1}{\bf e}}} \ \ & {\rm if}  \ \ d>0 \\  \frac{1}{2}\sqrt{{\bf g}^{T}{\bf E}^{-1}{\bf g}}  \ \ \ \ \ \ &{\rm otherwise}
\end{cases}
$$
Similarly, following the same analysis, we obtain:
$$
\sup_{\tilde{\bf w}} -f(\tilde{\bf w}) = \begin{cases} \ \ \frac{1}{2}\sqrt{{\bf g}^{T}{\bf E}^{-1}{\bf g}+\frac{d^2}{b-{\bf e}^{T}{\bf E}^{-1}{\bf e}}}  \ \ &{\rm if}  \ \ d<0 \\  \frac{1}{2}\sqrt{{\bf g}^{T}{\bf E}^{-1}{\bf g}}  \ \ \ \ \ \ &{\rm otherwise}
\end{cases}
$$
Comparing the optimal objective values, at optimum we have:
\begin{align*}
\tilde \bw^\star= \frac{b-\be^T\bE^{-1}\be}{|d|}\bE^{-\frac{1}{2}} \bg \ \ \ {\rm if} \ \  d\neq0 
\end{align*}
Going back to $\bw$, we ultimately find that the optimal $\bw^\star$ has the following closed-form expression
$$
\bw^\star= \bE^{-\frac{1}{2}} \tilde\bw^\star-\bE^{-1}\be=\bE^{-1}\left[\frac{b-\be^T\bE^{-1}\be}{|d|}\bg-\be\right].
$$

\bibliographystyle{IEEEbib}
{\bibliography{IEEEabrv,IEEEconf,ref}}

\end{document}